\documentclass[letter, 10pt, conference]{ieeeconf}      
\makeatletter
\def\endthebibliography{%
  \def\@noitemerr{\@latex@warning{Empty `thebibliography' environment}}%
  \endlist
}
\makeatother

\IEEEoverridecommandlockouts 
\let\labelindent\relax
\usepackage{amsmath}
\usepackage{amsfonts}
\usepackage{nicematrix}
\usepackage{textgreek} 
\usepackage[hidelinks]{hyperref}
\usepackage{cleveref}
\usepackage{relsize}
\usepackage{graphicx}
\usepackage{soul}
\usepackage{bm}
\usepackage{enumitem}
\usepackage{commath}
\usepackage{graphicx}
\graphicspath{{figures/}}
\usepackage{amsmath}
\usepackage{subcaption}
\usepackage{breqn}
\usepackage{booktabs}
\usepackage{empheq}
\usepackage{amsfonts}
\usepackage{amssymb}

\usepackage{amsthm}
\usepackage{mathrsfs}
\usepackage{accents}
\usepackage{thmtools}
\usepackage{thm-restate}
\usepackage{float}
\usepackage{comment}
\usepackage{algpseudocode}
\usepackage{algorithm}
\usepackage{xspace}
\usepackage{breakurl}

\usepackage{lipsum}

\usepackage{tikz}
\usepackage{pgfplots}
\usepackage{siunitx}
\pgfplotsset{compat=newest}
\usepgfplotslibrary{units}

\usepackage{cuted}
\newcommand{\arxivVersion}[1]{#1}
\newcommand{\CDCversion}[1]{}

\theoremstyle{plain}
\newtheorem{theorem}{Theorem}
\newtheorem{lemma}{Lemma}

\newtheorem{assumption}{Assumption}

\newtheorem{proposition}{Proposition}

\newtheorem{problem}{Problem}
\newtheorem*{problem*}{Problem}
\newtheorem*{theorem*}{Theorem}
\newtheorem{assumption*}{Assumption}

\newcommand{\diag}{\mathrm{diag}}

\Crefname{equation}{Equation}{Equations}
\crefname{equation}{}{}
\Crefname{figure}{Figure}{Figures}
\crefname{figure}{Fig.}{Figs.}
\crefname{table}{Tab.}{Tabs.}
\Crefname{table}{Table}{Tables}
\crefname{section}{Sec.}{Secs.}
\Crefname{section}{Sec.}{Secs.}
\crefname{problem}{Problem}{Problems}
\Crefname{problem}{Problem}{Problems}
\crefname{definition}{Definition}{Definitions}
\Crefname{definition}{Definition}{Definitions}
\crefname{lemma}{Lm.}{Lms.}
\Crefname{lemma}{Lemma}{Lemmas}
\crefname{theorem}{Thm.}{Thms.}
\Crefname{theorem}{Theorem}{Theorems}
\crefname{remark}{Rmk.}{Rmks.}
\Crefname{remark}{Remark}{Remarks}
\crefname{enumi}{item}{items}
\Crefname{enumi}{Item}{Items}
\crefname{algocf}{Alg.}{Algs.}
\Crefname{algocf}{Algorithm}{Algorithms}
\crefname{assumption}{Asm.}{Asms.}
\Crefname{assumption}{Assumption}{Assumptions}
\Crefname{Property}{Property}{Properties}
\crefname{Property}{Property}{Properties}

\usepackage[backend=biber,giveninits=true,sorting=none]{biblatex}
\newcommand{\rvec}{\mv{r}}
\newcommand{\omegades}{\omega_{\mathrm{des}}}
\newcommand{\omegadesvec}{\mv{\omega}_{\mathrm{des}}}
\newcommand{\dt}{\delta}

\newcommand{\N}{\ensuremath{\mathbb{N}}}
\newcommand{\R}{\ensuremath{\mathbb{R}}}

\newcommand{\Interior}[1]{\ensuremath{ \text{Int}(#1) }}

\newcommand{\mat}[1]{ \begin{pNiceMatrix} #1 \end{pNiceMatrix}  }

\newcommand{\nS}{9}
\newcommand{\nI}{6}
\newcommand{\nthruster}{n_{th}}

\newcommand{\Rot}{ \ensuremath{\mathnormal{R}} }

\newcommand{\Rotfull}{ \tilde{R}  }
\newcommand{\Rotfullidx}[1]{ \tilde{R}_{#1}  }

\newcommand{\traj} {\ensuremath{\chi}}
\newcommand{\utrajmax}{u_{\mathrm{max}}}

\newcommand{\M}{ \ensuremath{M} }
\newcommand{\Minv}{ \ensuremath{M^{-1}} }
\newcommand{\fe}{ \ensuremath{f(e_{\omega, k})} }
\newcommand{\feidx}[1]{f(e_{\omega, #1})}

\newcommand{\fmax}{F_{\mathrm{max}}}
\newcommand{\Fphysvec}{ F^{\mathrm{th}} }
\newcommand{\fvirtvec}{ \ensuremath{ F_v } }

\newcommand{\fvirt}{ \ensuremath{\Vert \fvirtvec \Vert} }

\newcommand{\U}{\ensuremath{\mathcal{U}}}
\newcommand{\Uphys}{\ensuremath{\mathcal{U}^{\mathrm{th}}}}
\newcommand{\Uphysidx}[1]{\ensuremath{\mathcal{U}^{\mathrm{th}}_{#1}}}
\newcommand{\Ures}{\ensuremath{\mathcal{U}_{\mathrm{res}}}}
\newcommand{\Uresidx}[1]{\ensuremath{\mathcal{U}_{\mathrm{res}, #1}}}

\newcommand{\X}{\ensuremath{\mathcal{X}}}
\newcommand{\Xterm}{\ensuremath{\mathcal{T}}} 

\newcommand{\x}[2][]{%
  \ifx\relax#1\relax
    x_{#2}%
  \else
    x_{#2|#1}%
  \fi
}

\newcommand{\e}[2][]{%
  \ifx\relax#1\relax
    e_{#2}%
  \else
    e_{#2|#1}%
  \fi
}

\newcommand{\ez}{\e{0}}

\newcommand{\rr}[2][]{%
  \ifx\relax#1\relax
    \traj_{#2}%
  \else
    \traj_{#2|#1}%
  \fi
}

\newcommand{\uu}[2][]{%
  \ifx\relax#1\relax
    u_{#2}%
  \else
    u_{#2|#1}%
  \fi
}

\newcommand{\fcn}[3][]{%
  \ifx\relax#1\relax
    #2_{#3}%
  \else
    #2_{#3|#1}%
  \fi
}

\newcommand{\termC}{l_{\mathcal{T}}}

\newcommand{\Qe}{Q_e}

\newcommand{\Qu}{Q_u}
\newcommand{\Qutilde}{\bar{Q}_u}
\newcommand{\Quhat}{\hat{Q}_u}

\newcommand{\Qeempc}{\hat{Q}_{\mathrm{e}} }
\newcommand{\Quempc}{\hat{Q}_{\mathrm{u}} }

\newcommand{\utrajk}{u_{\mathrm{ref}}^k}

\newcommand{\utrajs}{u_{\mathrm{ref}}^s}

\newcommand{\ufblin}{u_{fl}^k}

\newcommand{\utermCeMPC}{u_{\hat{\mu}}^k}

\newcommand{\eom}{e_{\omega}}
\newcommand{\eomidx}[1]{e_{\omega, #1}}
\newcommand{\controlEmpc}{\hat{\mu}}
\newcommand{\controlEmpck}[1]{\hat{\mu}_{k}}

\newcommand{\contrMPC}{\mu^{\mathrm{MPC}}}
\newcommand{\contrTerm}{\mu^{\mathrm{term}}}

\newcommand{\K}{K}
\newcommand{\kone}{\kappa_1}
\newcommand{\ktwo}{\kappa_2}
\newcommand{\kthree}{\kappa_3}
\newcommand{\ki}[1]{\kappa_{#1}}

\newcommand{\Uempc}{ \hat{\mathcal{U}} }
\newcommand{\Uempcupper}{ \check{\U} }
\newcommand{\costEmpc}{\hat{l}}
\newcommand{\infsum}[1]{\ensuremath{ \sum_{k=0}^{\infty} #1 } }
\newcommand{\empc}{\hat{e}}
\newcommand{\uempc}{\hat{u}}
\newcommand{\PtermEmpc}{\hat{P}}
\newcommand{\Xtempc}{\hat{\X}}

\newcommand{\E}{\mathcal{E}}
\newcommand{\Ufl}{\mathcal{U}_{\mathrm{fl}}}
\newcommand{\Utraj}{\mathcal{U}_{\traj}}
\newcommand{\XempcFeasible}{\hat{\X}_f}

\newcommand{\rempc}{\hat{u}_{\max}}
\newcommand{\emax}{e_{\max}}

\newcommand{\fblinmax}{f_{\mathrm{max}}}

\newcommand{\Fu}{F_{\mathrm{u}}}
\newcommand{\Fuidx}[1]{F_{\mathrm{u}, #1}}

\newcommand{\classkappainf}{class-$\mathcal{K}_\infty$\xspace}

\newcommand{\pedroNotation}{} 

\ifdefined\pedroNotation
\renewcommand{\omegades}{\omega_{\mathrm{d}}}
\renewcommand{\omegadesvec}{\boldsymbol{\omega}_{\mathrm{d}}}
\renewcommand{\Fphysvec}{f_{\circ}}
\renewcommand{\Uphysidx}[1]{\ensuremath{\mathcal{F}^{\mathrm{th}}_{[#1]}}}
\renewcommand{\Uphys}{\ensuremath{\mathcal{F}^{\mathrm{th}}}}
\renewcommand{\fmax}{f^{\mathrm{max}}}
\renewcommand{\mat}[1]{ \begin{bmatrix} #1 \end{bmatrix}  }
\renewcommand{\Ures}{\ensuremath{\mathcal{U}_{\mathrm{r}}}}
\renewcommand{\Uresidx}[1]{\ensuremath{\mathcal{U}_{\mathrm{r},[#1]}}}
\renewcommand{\fvirtvec}{ \ensuremath{ u_v } }

\renewcommand{\Fu}{u_{o}}
\newcommand{\Fuk}{u_{o,k}}
\renewcommand{\Fuidx}[1]{u_{o, #1}}
\newcommand{\scu}{u_r}
\newcommand{\scuidx}[1]{\ensuremath{u_{r,[#1]}}}
\renewcommand{\fe}{ \ensuremath{\mathrm{f}(e_{\omega})} }
\renewcommand{\feidx}[1]{\ensuremath{\mathrm{f}(e_{\omega, #1})}}
\renewcommand{\rvec}{\ensuremath{r_{\scriptscriptstyle\Diamond}}}
\renewcommand{\omegadesvec}{{\omega}^{\scriptstyle\circ}_{\mathrm{des}}}

\fi

\title{\LARGE \bfseries Fault-tolerant Model Predictive Control for Spacecraft}

\author{Raphael Stöckner\textsuperscript{\dag}, Pedro Roque\textsuperscript{\dag}, Maria Charitidou\textsuperscript{\ddag} and Dimos V. Dimarogonas\textsuperscript{\dag}%
\thanks{\textsuperscript{\dag} Raphael Stöckner, Pedro Roque and D. V. Dimarogonas are with the Division of Decision and Control Systems, KTH Royal Institute of Technology, Stockholm, Sweden. E-Mail: {\tt\small $\{$stockner, padr, dimos$\}$@kth.se} }%
\thanks{\textsuperscript{\ddag} Maria Charitidou is with the Institute for Systems Research, University of Maryland College Park, USA. E-Mail: {\tt\small mchar@umd.edu} }%
\thanks{This work was supported by the H2020 ERC Grant LEAFHOUND, the Swedish Foundation for Strategic Research (SSF), the Swedish Research Council (VR), the Knut och Alice Wallenberg Foundation (KAW), and the WASP NEST DISCOWER Project.}%
}

\makeatletter
\def\footnoterule{\relax%
  \kern-5pt
  \hbox to \columnwidth{\vrule width 0.5\columnwidth height 0.4pt\hfill}
  \kern4.6pt}
\makeatother

\begin{document}

\maketitle
\thispagestyle{empty}
\pagestyle{empty}

\begin{abstract}
Given the cost and critical functions of satellite constellations, ensuring mission longevity and safe decommissioning is essential for space sustainability. This article presents a Model Predictive Control for spacecraft trajectory and setpoint stabilization under multiple actuation failures. The proposed solution allows us to efficiently control the faulty spacecraft enabling safe navigation towards servicing or collision-free trajectories. The proposed scheme ensures closed-loop asymptotic stability and is shown to be recursively feasible. We demonstrate its efficacy through open-source numerical results and realistic experiments using the ATMOS platform.
\end{abstract}

\section{Introduction}
In recent years, the safety of autonomous systems under failures has experienced increasing attention from the research community \cite{electronics9091513}. 
Spacecraft failures are particularly critical, since dysfunctional vehicles risk contributing to space debris \cite{SHAN201618} and potential Kessler syndrome.
Thus, precautionary measures are needed to prevent the creation of space debris as well as to ensure that a damaged spacecraft can safely return to service stations or deorbit.

Most underactuated spacecraft research focuses on attitude stabilization. Theoretical analysis shows that thruster-based attitude stabilization needs two or more independent paired thrusters \cite{Crouch1984} and can not be achieved with a smooth controller. In \cite{CoverstoneCarroll1996} attitude reorientation has been ensured for underactuated spacecraft using a variable structure controller for detumbling and a series of rotations.
Newer results propose switching control schemes \cite{Casagrande2008} or perturbed feedback linearization \cite{Hall2010} for global attitude stabilization.
These results assume decoupled attitude and position control, albeit this can yield infeasible inputs and despite possible loss of (small-time local) controllability for the combined position-attitude system due to actuator failures \cite{Pong2011}.

In \cite{Pong2011,Shen2015,Shen2018,Han2015} Model Predictive Control schemes (MPC) have been proposed for spacecraft control under actuator failures. In \cite{Shen2015,Shen2018} authors consider various types of failures such as effectiveness or full actuator loss as well as bias faults but assume that the system remains fully actuated. The work in \cite{Han2015} considers full actuator failures without any input constraints. To the best of our knowledge, \cite{Pong2011} is the only work that controls both the spacecraft's position and attitude under actuator failures in presence of input constraints and despite underactuation. The approach uses MPC for setpoint stabilization with guaranteed stability but employs a zero terminal set which can potentially decrease the size of the region of attraction.

In this work, an MPC control law is proposed for trajectory tracking of spacecraft in presence of multiple actuator failures and input constraints. 
Inspiration is drawn from works on control of underactuated quadcopters, where the control of some variables is dropped in order to keep others in a safe set \cite{Mueller2014,Nan2022}. The resulting controller ensures that the spacecraft  converges asymptotically to a small-scale orbit designed to ensure partial compensation of fault-driven forces. Contrary to state of the art, the stability of the closed loop system is guaranteed by novel terminal ingredients that are based on a feedback linearizing terminal controller and (linear) explicit MPC (eMPC) \cite{Alessio2009}. The proposed approach leads to larger terminal sets compared to the standard LQR method and provides more flexibility on the design of the terminal set allowing us to easily adjust the set's size along desired directions while respecting the system's constraints.

\subsubsection*{Notation}
We use $\N$ and $\R$ for natural and real numbers and $\R_{\geq 0}$ ($\R_{> 0}$) for non-negative (positive) reals.
The operator $\diag(\ldots)$ creates a diagonal matrix from given arguments and 
$\lambda(A)$ denotes an eigenvalue of $A \in \R^{n \times n}$.
$A \succ 0$ ($A \succeq 0$) means $A$ is symmetric positive (semi-) definite.
Row ranges of vectors, matrices or sets are indexed as $X_{[i:j]} = [x_i^T, ..., x_j^T]^T \in \R^{(j-i+1) \times m}$ for a matrix $X = [x_1^T, ..., x_i^T, ..., x_j^T, ... , x_{n}]^T \in \R^{n \times m}$ with rows $x_{(\cdot)} \in \R^{1 \times m}$.
The submatrix $X_{[i:j,k:l]}\in \R^{(j-i+1) \times (l-k+1)}$ denotes $i$-th to $j$-th rows and $k$-th to $l$-th columns of $X$.
The matrix $x^{\times}$ is defined s.t. $x^{\times}y=x \times y ~\forall x, y \in \R^3$.
We denote $\mathcal{U} \oplus \mathcal{V}$ and $\mathcal{U} \ominus \mathcal{V}$ as the Minkowski sum and the Minkowski (Pontryagin) difference of sets $\mathcal{U}, \mathcal{V}$. 
The interior of $\mathcal{U}$ is $\Interior{\mathcal{U}}$ and the cartesian product is $\mathcal{U} \times \mathcal{V}$.
For brevity, we write $\sum_{0 \leq i,j \leq N}(\cdot) = \sum_{i=0}^N \sum_{j=0}^N (\cdot)$. 
A \textit{\classkappainf function} $\alpha:\mathbb{R} \to \mathbb{R}$ is continuous, strictly increasing, $\alpha(0) = 0$ and $\lim_{r\to \infty} \alpha(r) = \infty$. 

\section{Background}
\label{sec:background}
In this work a thruster-controlled spacecraft 
is considered whose dynamics are given
as follows:
\begin{equation}
\label{eq:dynamics_continuous}
\begin{aligned}
    \dot{p}_r &= v_r, \\
    \dot{v}_r &= m^{-1} \Rot^T f_r,  \\
    \dot{q} &= \frac{1}{2} \Omega(\omega) q, \\
    \dot{\omega} &= J^{-1} (-\omega^{\times} J \omega + \tau ),\\ 
\end{aligned}
\end{equation}
where $p_r, v_r \in \R^3$, are the position and velocity of the spacecraft, respectively, $q \in \R^4$ the attitude quaternion in global coordinates, $\omega \in \R^3$ the angular velocity in body coordinates, $m$ the mass of the spacecraft, $J = \diag(j_0, j_1, j_2)$ the inertia matrix expressed in a principal axis system, $\Rot = \Rot(q) \in SO(3)$ the rotation matrix w.r.t. a global coordinate frame and 
$
    \Omega(\omega) = \left[ \begin{smallmatrix}
        \omega^{\times} & \omega \\
        -\omega^T & 0
    \end{smallmatrix} \right].
$
The input of the system is the force $f_r\in\R^3$ and torque $\tau\in\R^3$ in the body frame, concatenated in $\scu = [f_r^T \ \ \tau^T]^T\in\R^6$. Spacecraft are actuated by $\nthruster$ thrusters, each producing a physical thruster force $f_i \in [0, \fmax] \in \R_{\geq 0},$  $i=0, \dots, \nthruster$. 
The actuation failures considered in this work are known thruster failures, resulting in the affected thruster's input becoming fixed at a constant value $f_i = f_i^{\text{fault}} \in [0, \fmax]$. For each thruster, we define the individual constraint set $\Uphysidx{i}$ as:
\begin{align}
    \Uphysidx{i} = 
        \begin{cases}
            \{f_i^{\text{fault}} \} & \text{if thruster }i \text{ fails} \\
            \lbrack 0, \fmax \rbrack & \text{otherwise}. 
        \end{cases}
\label{eq:ind_input}
\end{align}
Based on \eqref{eq:ind_input}, we define the input set of $\Fphysvec=\mat{f_1^T &\dots & f^T_\nthruster}^T$ as $ \Uphys = \Uphysidx{1} \times \ldots \times \Uphysidx{\nthruster}.$ 

The system input $\scu$ in \cref{eq:dynamics_continuous} relates to $\Fphysvec$ through $\scu = D \Fphysvec$, where the allocation matrix $D \in \R^{6 \times \nthruster}$ is determined by the thruster positions. As $\Uphys$ is a compact, convex polyhedron, so is $\Ures = D \Uphys$. 
Here, we make the following assumption on the 
controllability of the spacecraft after the faults occur:
\begin{assumption} \label{as:zero_in_interior}
There exists a vector $\fvirtvec = \mat{f^T & 0_{1 \times 3}}^T\in \R^6$ with $f \in \R^3$ such that $\fvirtvec \in \text{Int}(\Ures)$ and $f$ is aligned to any of the principal axes of the system.
\end{assumption}
From \cref{as:zero_in_interior}, without loss of generality, the body coordinates are chosen such that a tentative $\fvirtvec$ is given by  $\fvirtvec = \mat{0 & \bar{f} & 0 & 0_{1 \times 3}}^T$, $\bar{f}>0$. 
This may require rotating from the original body coordinates to new ones. In this case, $\Ures$ (defined in body coordinates) is rotated similarly.
\Cref{as:zero_in_interior} allows setting input torque $\tau$ to zero, but the same does not necessarily hold for the input force $f$ (i.e., failures may cause continuous spacecraft acceleration). However, by controlling the angular velocity, the system can follow a small local orbit where the continuous acceleration keeps the system on the orbit.

Despite the faults, the system needs to ``robustly" track a
desired reference trajectory $x_{\text{ref}}:\R_{\geq 0}\rightarrow \R^{13}$ satisfying:
\begin{assumption} \label{as_reftraj} \label{as:traj_feasible}
    The reference trajectory $x_{\text{ref}}:\R_{\geq 0}\rightarrow \R^{13}$ is defined as 
    $x_{\text{ref}}(t)=x_{\text{ref}}^k,  t\in [t_k, t_{k+1}),$
    where $t_k:=k\dt, k\in \N,$ $\dt>0$ is the sampling period and 
    $x_{\text{ref}}^k=\mat{\chi_{k}^T & q_{\mathrm{ref}, k}^T & \omega_{\mathrm{ref}, k}^T}^T \in \R^{13}$,
    is the reference state at the $k$-th sampling instant that evolves according to:
    \begin{equation}
    x_{\text{ref}}^{k+1}=G(x_{\text{ref}}^k,\utrajk),   \label{eq:G_fuc}
    \end{equation}
    where $G:\R^{13}\times \R^6 \rightarrow \R^{13}$ is a continuously differentiable function, $\chi_{k} = \mat{p_{des,k}^T & v_{des,k}^T}^T$ is a concatenation of the desired position $p_{des,k} \in \R^3$ and velocity $v_{des,k}\in\R^3$, and $\utrajk,k\in \N$ is a known input sequence satisfying $\utrajk \in \Utraj := \{u | u^Tu \leq \utrajmax \},$ for every $k\in \N,$ where $\utrajmax \in \R_{>0}$ is chosen to ensure that $\Utraj \subset \Ures \ominus \fvirtvec.$
\end{assumption}
\Cref{as_reftraj} introduces a piece-wise continuous trajectory $x_{\text{ref}}(t)$ with the corresponding reference input satisfying the input constraints of the spacecraft. 

Based on the above, we can formally express the problem considered in this work as follows:
\begin{problem} \label{prob:main}
Consider the spacecraft dynamics  \eqref{eq:dynamics_continuous} that are subject to 
input constraints defined in 
\eqref{eq:ind_input}.
Let further \cref{as:zero_in_interior,as_reftraj} hold. Design a control law $\scu:\R_{\geq 0} \times \R^{13} \rightarrow \Ures$ such that, as $t\rightarrow\infty$, $\Vert p_r(t)-\traj_{[1:3]}(t)\Vert \leq r_s,$
where $r_s \in \R_{>0}$ is a tuning parameter.
\end{problem}

\section{Control Design}
\label{sec:control}
In this section, we design a reference tracking MPC accounting for the actuator failures of the spacecraft. 
In \cref{sec:orbiting} we propose a coordinate transformation and introduce a modified system, called from now on the \textit{orbit system}.
Then, in \cref{sec:mpc}, we design the MPC for the orbit system to ensure ``robust" tracking.

\subsection{Orbit Dynamics} 
\label{sec:orbiting}

\begin{figure}[tb]
    \centering
    \usetikzlibrary{arrows.meta}
\usetikzlibrary{calc}

\newcommand{\squareSize}{1}
\newcommand{\tiltAngle}{65}
\newcommand{\squareCenterX}{4.5}
\newcommand{\squareCenterY}{1.0}
\newcommand{\axisLength}{0.8}
\newcommand{\distanceToC}{1.8}
\newcommand{\arrowBLength}{1.5}
\newcommand{\arrowHeadSize}{4pt} 

\begin{tikzpicture}[
    >={Triangle[length=\arrowHeadSize,width=\arrowHeadSize]}, 
    font=\sffamily
]
    \draw [->] (0,0) -- (\axisLength,0) node [midway, below] {$x$};
    \draw [->] (0,0) -- (0,\axisLength) node [midway, left] {$y$};
    
    \coordinate (SquareCenter) at (\squareCenterX, \squareCenterY);
    \draw [rotate around={\tiltAngle:(SquareCenter)}] ($(SquareCenter)+(-\squareSize/2,-\squareSize/2)$) 
          rectangle 
          ($(SquareCenter)+(\squareSize/2,\squareSize/2)$);
    
    \draw [rotate around={\tiltAngle:(SquareCenter)}, ->] (SquareCenter) -- ($(SquareCenter)+(\axisLength,0)$) ; 
    \draw [rotate around={\tiltAngle:(SquareCenter)}, ->] (SquareCenter) -- ($(SquareCenter)+(0,\axisLength)$) ; 
    
    \coordinate (PointC) at ($(SquareCenter)+({-\distanceToC*sin(\tiltAngle)},{\distanceToC*cos(\tiltAngle)})$);
    \draw [dashed] (SquareCenter) -- (PointC) node [pos=0.60, above right] {$r$};
    
    \draw[black,fill=black] (PointC) circle (.5ex); 
    
    \draw [->] (0,0) -- (SquareCenter);
    
    \draw [rotate around={\tiltAngle:(SquareCenter)}, ->] ($(SquareCenter)-(0,\arrowBLength)$) -- (SquareCenter) node [pos=0.15, below left] {$\fvirt$};
    
    \node at ($0.63*(SquareCenter)+(0,0.1)$) [above left, 
           inner sep=0.1pt,
           column sep=0.5em,
           row sep=0.5em,
           scale=0.9]
    { $p_r$ };
  
    \node at ($0.7*(PointC)+(-0.2,0)$) [above left, 
       inner sep=0.1pt,
       column sep=0.5em,
       row sep=0.5em,
       scale=0.9]
    { $c$ }; 
    \draw[->] (0,0) to (PointC);

    \draw[dashed] ([shift=(20:\distanceToC))]PointC) arc (20:-200:\distanceToC);
    \draw[->] ([shift=(\tiltAngle-90:0.35))]PointC) arc (\tiltAngle-90:290:0.35) node [above right,pos=0.5] {$\omegades$};
 
\end{tikzpicture}
    \caption{Schematic of the orbiting system. The robot orbits with angular velocity $\omegades$ along the dashed orbit with radius $r$. By \cref{eq:center_p}, the robot and orbit positions are directly coupled, so the robot is always on the orbit. The controller directly controls the orbit center $c$ to track a trajectory.
    }
    \label{fig:micro_orbit}
\end{figure}
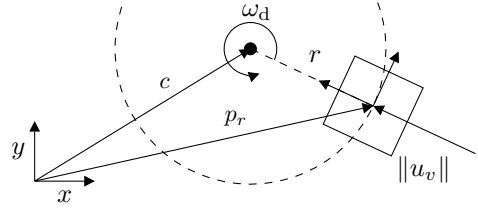

We start by designing the orbit dynamics. By \cref{as:zero_in_interior}, there exists a vector $\fvirtvec = \mat{ f^T & 0_{1 \times 3}}^T$ with $f=\mat{0 & \bar{f} & 0}^T$ aligned with the y-axis of the body frame and $\bar{f} = \mathrm{const} \in \R_{>0}$. This force, together with the centrifugal force of the spacecraft is used to bring the spacecraft on a stable orbit. 

First, introduce $\rvec = \mat{0&r&0}^T=\mathrm{const}$ and design the orbit with radius $r \in \R_{>0}$ such that the orbit center and velocity $c,v\in \R^3$ in global coordinates relate to $p_r$ through
\begin{align}
    c=p_r+R^T \rvec.  \label{eq:center_p}
\end{align}
    Then, choose $\omegades = \mathrm{const.}$
    to satisfy
    \begin{align}
        f = m\omegades^2 \rvec.  \label{eq:condition}
    \end{align}
    By \cref{eq:center_p}, the orbit center is fixed to the body frame. In this way, if the spacecraft is on  orbit, then robust trajectory tracking can be ensured by appropriately controlling the orbit dynamics. Eq. \cref{eq:condition} ensures that $f$ is aligned with $\rvec$ and always points to the orbit center as shown in  \cref{fig:micro_orbit}. In addition, it constrains the selection of $r$ and $\omegades$ since choosing $r \rightarrow 0$ implies $\omegades \rightarrow \infty.$
Let the input of the orbit system $\Fu \in \U \subset \R^6$ be a vector related to $\scu \in \Ures$ in \eqref{eq:dynamics_continuous} through
$
    \scu = \fvirtvec + \Fu,       
$
which is used to derive $\U$ from $\Ures$ as
\begin{align}
\U := \Ures \ominus \fvirtvec.
 \label{eq:u_def}
\end{align}
Differentiating \eqref{eq:center_p} two times, using $\frac{d}{dt}\Rot^T = \Rot^T \omega^{\times}$ and \eqref{eq:dynamics_continuous}, we obtain the orbit dynamics:
\begin{equation}
\label{eq:orbit_dynamics}
\begin{aligned}
    \dot{c} &= v_r + \Rot^T \omega^{\times} \rvec := v, \\
    \dot{v} &= \dot{v}_r + \Rot^T (\dot{\omega}^{\times} \rvec + \omega^{\times} \omega^{\times} \rvec) \\
    &= \Rot^T(\omega^{\times} \omega^{\times} \rvec + [J^{-1} (-\omega^{\times} J \omega)]\times \rvec + \M_{[1:3]} \fvirtvec) \\
    & \qquad + \Rot^T \M_{[1:3]} \Fu,\\
    \dot{q}&=\frac{1}{2}\Omega(\omega)q, \\
\dot{\omega}&=-J^{-1}\omega^{\times} J \omega + \M_{[4:6]} \Fu, 
\end{aligned}
\end{equation}
where $\M \in \R^{6 \times 6}$ is an invertible matrix defined as
\begin{align}
    \M &= \begin{bmatrix}
        \frac{1}{m} I_{3} & 
            \left[ \begin{smallmatrix}
                0 & 0 & -\frac{r}{j_2} \\
                0 & 0 & 0 \\
                \frac{r}{j_0} & 0 & 0
            \end{smallmatrix} \right]
         \\
        0_{3 \times 3} & J^{-1}
    \end{bmatrix}, 
    \label{eq:def_M}
\end{align}
satisfying $\mat{m^{-1} \scuidx{1:3} + [(J^{-1} \scuidx{4:6}) \times \rvec] \\
J^{-1} \scuidx{4:6}} = \M \scu$.
Notice that for $\omega = \mat{ 0 & 0 & \omegades}^T =: \omegadesvec$, the orbit acceleration in \eqref{eq:orbit_dynamics} becomes  $\dot{v} = \Rot^T \mat{0 & \frac{1}{m} \fvirt -r\omegades^2 & 0}^T + \Rot^T \M_{[1:3]} \Fu = \Rot^T \M_{[1:3]} \Fu$ using the choice from \cref{eq:condition}.
Therefore, the dynamics and parameter choices of the orbit itself ensure a partial error compensation, and due to $\fvirtvec \in \Interior{\Ures} \Rightarrow 0 \in \Interior{\U}$, \cref{eq:orbit_dynamics} is controllable.

The orbit dynamics \eqref{eq:orbit_dynamics} are discretized using Euler discretization, yielding the discrete-time dynamics:
\begin{align} 
\bar{x}_{k+1} &= \bar{x}_k + \dt \mat{ \dot{c}^T & \dot{v}^T & \dot{q}^T & \dot{\omega}^T }^T :=\bar{G}(\bar{x}_k,u_k),
\label{eq:dynamics}
\end{align}
where $\bar{x}_k:=\mat{c_k^T & v_k^T &q_k^T & \omega_k^T}^T,$ $u_k:=\Fuk$ is the state and input of the orbit dynamics at the $k$-th sampling instant, respectively and  $\dt>0$ is the same as in \cref{as_reftraj}. For the remainder of this work, we limit our attention to a class of reference trajectories for which the following holds:
\begin{assumption}\label{as:traj_dynfeas}
The function $G:\R^{13}\times \R^6 \rightarrow \R^{13}$ in \eqref{eq:G_fuc} satisfies
$
G(y_1, y_2) = \mat{\bar{G}_{[1:10]}(y_1, \Rotfull(y_{1, [7:10]}) y_2) \\ \omegadesvec},
$
for every $(y_1,y_2)\in \R^{13}\times \R^6,$ where $\bar{G}:\R^{13}\times \R^6 \rightarrow \R^{13}$ is defined in \eqref{eq:dynamics} and $\Rotfull := \left[\begin{smallmatrix}
    \Rot & 0_{3 \times 3} \\
    0_{3 \times 3} & I_3
\end{smallmatrix} \right].$

\end{assumption}

Based on \eqref{eq:dynamics} and \cref{as_reftraj,as:traj_dynfeas} we introduce the tracking error variables at time step $k$ as
$
    e_{p,k} = \mat{ c_k^T & v_k^T}^T - \traj_k, \eomidx{k} = \omega_k - \omegadesvec
$
where $e_{p,k}$ is the orbit center and velocity tracking error, respectively, and $\eomidx{k}$ is the angular velocity error. Let $e_k:=\mat{e_{p,k}^T & \eomidx{k}^T}^T.$ Then, the error dynamics are defined according to \cref{eq:orbit_dynamics} and \cref{as_reftraj,as:traj_dynfeas} as follows:
\begin{subequations}
\label{eq:err_dyn_orbiting}
\begin{align}
    e_{k+1} &= \underbrace{e_k +  
        \dt \mat{
            e_{k, [4:6]} \\
            \Rotfull^T (\feidx{k} + \M \Fuidx{k}) - \M \utrajk
        }}_{:=\phi(e_k, \Fuidx{k})}, \label{eq:errp}
\\
    q_{k+1} &= q_k + \frac{1}{2} \dt \Omega(\omegadesvec + \eomidx{k}) q_k := \varphi(q_k, e_k), \label{eq:errq}
\end{align}
\end{subequations}
where $\Rotfull$ is the same as in \cref{as:traj_dynfeas}
and $\feidx{k}$ is defined as:
{
\newcommand{\ombar}{\bar{\omega}}
\begin{align}
    \feidx{k}= \mat{
        \ombar_k^{\times} \ombar_k^{\times} \rvec + [J^{-1} (-\ombar_k^{\times} J \ombar_k)]\times \rvec + \frac{1}{m} \fvirtvec \\
        -J^{-1}\ombar_k^{\times} J \ombar_k 
    }, \label{eq:def_feom}
\end{align}
with $\ombar_k := \eomidx{k} + \omegadesvec$. Notice that $\Vert \mathrm{f}(0) \Vert = 0$ using \cref{eq:condition}.
}

\begin{assumption} \label{as:discretization_error}
For a sampling time $\dt > 0$, there exists a $\beta = \beta(\dt) > 0,$ such
$\Vert c(t)-c_k \Vert \leq \beta, ~\forall t \in [t_k,t_{k+1}), k\in \N$.
\end{assumption}
\noindent A bound $\beta$ can be obtained using \cite[Thm. 3.4]{khalil2002nonlinear} as in \cite{roque2022corridor}.

\subsection{Model Predictive Control}
\label{sec:mpc}

In the following, we propose an MPC controller for the error dynamics defined in \cref{eq:err_dyn_orbiting}. The controller employs novel terminal ingredients, partially based on eMPC, allowing for an enlarged terminal set as will be shown in \cref{sec:application}.

We propose the trajectory tracking controller for \cref{eq:err_dyn_orbiting} as
\begin{align}
\contrMPC_k=\fcn{u^{*}}{o,k} + \Rotfull \utrajk,
\label{eq:u_mpc}
\end{align}
where $\Rotfull$ is the rotation matrix from \cref{as:traj_dynfeas} evaluated at $q_k$, and $\utrajk$ the reference input defined in \cref{as_reftraj}.
The input $\fcn{u^{*}}{o,k}$ is obtained by solving \cref{eq:mpc} at each sampling time, resulting in $\mathbf{u}_{o,k} = \{ u_{o,0}, \dots, u_{o,N-1} \}$, and setting $\fcn{u^{*}}{o,k} := u_{o,0}$.
\begin{subequations}
\label{eq:mpc}
\begin{alignat}{2}
    &\underset{\mathbf{u}_{o,k}}{\text{min.}} ~ &&\sum_{s=0}^{N-1} \left[ e_s^T \Qe e_s + u_{o,s}^T \Qu u_{o,s} \right] + \termC (e_N) \\
    & ~ \text{s.t.} ~~~~
       && e_{s+1}     = \phi(e_{s}, u_{o,s} + \Rotfullidx{s} \utrajs), \\ 
    &  && q_{s+1} = \varphi(q_s, e_s), \\
    &  && u_{o,s} + \Rotfullidx{s} \utrajs    \in \U, \\
    &  && e_N       \in \Xterm     \\
    &  && \ez         = e_k, q_0 = q_k 
\end{alignat}
\end{subequations}
\noindent for $s = 0, \ldots, N-1$, 
where $\phi(e_k, \Fuidx{k}),  \varphi(q_k, e_k)$ are defined 
as in \cref{eq:err_dyn_orbiting}, $\Qe \in \R^{\nS \times \nS}$ and $\Qu \in \R^{\nI \times \nI}$, $\Qe, \Qu \succ 0$ are diagonal weight matrices, $N \in \N$ is the control horizon and $\termC (e_N)$ and $\Xterm$ are the terminal cost and terminal set, respectively, designed next.

figures/
\section{Theoretical Analysis}
\label{sec:analysis}
We analyze the stability of \eqref{eq:dynamics_continuous} with control \eqref{eq:u_mpc} using a terminal controller proposed in \cref{sec:analysis:term_contrl}, derive terminal ingredients in \cref{sec:analysis:term_set,sec:analysis:term_cost}, and prove asymptotic stability of the orbit dynamics and bounded spacecraft tracking error in \cref{sec:analysis:stability}.

\subsection{Terminal controller} \label{sec:analysis:term_contrl}
We begin the analysis by the design of a terminal controller for the orbit error dynamics in \eqref{eq:err_dyn_orbiting}. Here, we propose a nonlinear feedback control law defined as:
\begin{align}
    \contrTerm(e_k,k)
    &= \Minv \left[ -\mathrm{f}(\eomidx{k}) + \mat{ \Rot \controlEmpc_k \\ - \K \eomidx{k} }  \right] + \Rotfull \utrajk,
     \label{eq:term_contr} 
\end{align}
where $\M$ and $\mathrm{f}(e_{\omega,k})$ are defined as in \eqref{eq:def_M} and \eqref{eq:def_feom}, respectively, $\controlEmpc_k$ is a linear control law designed based on the eMPC scheme introduced later in \eqref{eq:empc_contr} and
the gain matrix $\K = \diag(\kone, \ktwo, \kthree) \in \R^{3 \times 3}$ with $0 < \ki{i} < \frac{1}{\dt}, i=1, 2, 3$ chosen to ensure that $\eom$ is asymptotically stable. 
The proposed terminal control law is defined as the sum of a nonlinear term, introduced to cancel the effect of $\mathrm{f}(\eom)$ in \eqref{eq:err_dyn_orbiting}, and an eMPC and linear control law, resp., that ensure asymptotic stability of $e_{p,k}$ and $\eomidx{k}$.
Substituting \eqref{eq:term_contr} to \eqref{eq:errp} yields:
\begin{subequations}
\label{eq: closedloop_term}
\begin{align}
e_{p,k+1} &=
    \hat{\phi}(e_{p,k}, \controlEmpc) 
   :=
    \mat{
        I_{3} & \dt I_{3} \\
        0_{3 \times 3} & I_{3}
    } e_{p,k} + \mat{
        0_{3 \times 3} \\ \dt I_{3} 
    } \controlEmpc_k,
    \label{eq:empc_sys}\\
 \eomidx{k+1} &= (I_{3} - \dt \K) \eomidx{k}.  \label{eq:omnew_ter} 
\end{align}
\end{subequations}
Finally, $\controlEmpc_k=\controlEmpc(e_k)$ in \eqref{eq:term_contr} is defined as $\controlEmpc(e_k)=\uempc_0^*,$ where $\uempc_0^*$ is the optimal solution of the multi-parametric QP program $P(e_k),$ with $P(e)$ defined as follows: 
\begin{equation}
\begin{aligned}
     &P(e): && \underset{\uempc_0, \ldots, \uempc_{\hat{N}-1}}{\text{min.}} ~&& \sum_{s=0}^{\hat{N}-1} \left[ \empc_s^T \Qeempc \empc_s + \uempc_s^T \Quempc \uempc_s \right] + \empc_{\hat{N}}^T \PtermEmpc \empc_{\hat{N}} \\
    & &&\quad ~ \text{s.t.}
        && \empc_{s+1} =\hat{\phi}(\empc_s, \uempc_s),
        ~s=0, \ldots, \hat{N}-1 \\
     & && && \uempc_s \in \Uempc, ~s=0, \ldots, \hat{N}-1 \\
     & && && \empc_{\hat{N}} \in \Xtempc, \\
     & && && \empc_0 = e
\label{eq:empc_contr}
\end{aligned}
\end{equation}
Here, $\hat{\phi}:\R^6 \times \R^3 \rightarrow \R^6$ is defined in \eqref{eq:empc_sys},
$\hat{N} \in \N$ is the planning horizon, $\Qeempc \in \R^{6 \times 6}$ and $\Quempc \in \R^{3 \times 3}$, $\Qeempc, \Quempc \succ 0$ are weight matrices,
$\Uempc \subset \U$ is a compact, polyhedral set  and $\PtermEmpc\succ 0,$ $ \Xtempc\subset \R^6$ are the terminal cost matrix and terminal set of \eqref{eq:empc_contr}, respectively, chosen as in \cite[Ch. 2.5.4]{rawlings2017model} to stabilize \eqref{eq:empc_sys}. 
The problem $P(e)$ can be solved parametrically and its solution can be calculated offline following \cite{Alessio2009}. The solution gives the optimal eMPC control law $\controlEmpc_k = \controlEmpc(\empc_k)$ and the parametric cost function $\costEmpc(\empc_0)$  in the set $\XempcFeasible= \{ e \in \R^6 | P(e) \; \text{is feasible} \} \subset \R^6$ where $P(e)$ has a recursively feasible, stabilizing solution. 
Since the problem is solved offline, the horizon $\hat{N}$ can be arbitrarily large, limited only by computational resources.

\subsection{Terminal set} \label{sec:analysis:term_set}
Next, given the terminal control law proposed in \cref{sec:analysis:term_contrl},  we design the terminal set $\Xterm$ of \eqref{eq:mpc} so that \cref{eq:errp} under \eqref{eq:term_contr} is asymptotically stable and  $(e_k,\contrTerm(e_k,k) \color{black})\in \Xterm \times \U,$ for every $k\in \N$.

We introduce the terminal set $\Xterm \subset \R^9:$
\begin{equation}
    \Xterm=\XempcFeasible \times \E,  \label{eq:terminal_set}
\end{equation}
where $\XempcFeasible \subset \R^6$, obtained in \cref{sec:analysis:term_contrl},  is the terminal set corresponding to $e_{p,k}$ and $\E \subset \R^3$ is the terminal set corresponding to $\eomidx{k}$, defined as
$
\E = \{ \eom \in \R^3 \vert -\emax \leq \eom \leq \emax\}, 
$
where the inequalities are defined elementwise and $\emax \in \R_{>0}^3$ is a vector appropriately chosen in the following. Let:
\begin{align}
\Uempcupper= \{ \uempc \in \R^3 | \Vert \uempc \Vert_2^2 \leq \rempc \}, \label{eq:uset_empc_ball}
\end{align}
where $\rempc>0$ is a parameter to be chosen such that $\Uempc \subseteq \Uempcupper \subset \U,$ $\Uempc \subset \R^6$ is the control set of the eMPC scheme defined in \eqref{eq:empc_contr} and $\U \subset \R^6$ is the control set of the orbit dynamics defined in \eqref{eq:u_def}. 

Our goal is to choose the tuning parameters $\emax\in \R^3$ and $\rempc\in \R_{\geq 0}$ such that the following are satisfied:

\begin{subequations}
\label{eq:term_cond}
\begin{align}
&e_k \in \Xterm \Rightarrow \contrTerm(e_k,k)\in \U , \quad \forall k\in \N, \label{eq:term_cond2} \\
&e_k \in \Xterm \Rightarrow e_{k+1}\in \Xterm , \quad \forall k\in \N, \label{eq:term_cond1}
\end{align}
\end{subequations}
where $e_{k+1}$ is defined according to \eqref{eq: closedloop_term} and $\contrTerm(e_k,k)$ is given in \eqref{eq:term_contr}. 
Equation \eqref{eq:term_cond2} guarantees that the terminal controller in \eqref{eq:term_contr} satisfies the desired input constraints and \eqref{eq:term_cond1} ensures that $\Xterm$ is forward invariant. 
These conditions will be used in \cref{sec:analysis:stability} to show the asymptotic stability of the orbit error dynamics. 

By continuity of \eqref{eq:def_feom} and compactness of $\E,$ for every $\eom \in \E,$ $\fe$ is bounded and so there exists a set: 
\begin{align}
    \Ufl &= \{ \nu \in \R^6 \vert \nu^T \nu \leq \fblinmax \}, \label{eq:feom_ball}
\end{align}
such that $\fe\in \Ufl,$ where  $\fblinmax$ is a parameter whose value depends on $\emax$. From \cref{as_reftraj}  $\utrajk \in \Utraj$ and using \eqref{eq:uset_empc_ball}, \eqref{eq:feom_ball} we can conclude that $\contrTerm(e,k)\in     \Minv [(\Uempcupper \times (-\K \E)) \oplus \Ufl \oplus \Utraj],$ for every $e\in \R^6 \times \E$ and $k\in \N.$ As a result, a sufficient condition for \eqref{eq:term_cond2} to hold is
$
    \Minv [(\Uempcupper \times( -\K \E)) \oplus \Ufl \oplus \M \Utraj] \subseteq \U.
    \label{eq:term_set}
$
In order to express this condition in terms of the design parameters $\emax, \rempc$ we first obtain an expression for $\fblinmax$ as the solution to 
\begin{equation}
\begin{aligned}
    &\underset{\eom\in \E}{\text{max.}} ~ && \fe^T \fe.
\label{eq:term_set_femax}
\end{aligned}
\end{equation}
With \cref{as_reftraj} and \eqref{eq:feom_ball} and after recalling that $\U$ is polyhedral due to \eqref{eq:ind_input} and \eqref{eq:u_def}, we can express  $\M \U \ominus \Utraj \ominus \Ufl$ as \cite{linke2015decomposition} 
$
    \M \U \ominus \Utraj \ominus \M \Ufl = \{x \in \R^6 |
    a_i x \leq b_i - \sqrt{\utrajmax} - \sqrt{\fblinmax}, i \in \mathcal{I}\},
$
where $\M \U := \{x \in \R^6 | a_i x \leq b_i, i \in \mathcal{I}\}.$

Using this, we can choose $\emax, \rempc$ as the optimal solution to the following optimization problem \cite[Ch. 8.4.2]{Boyd2004} that maximizes the hypervolume of $\Uempcupper \times (-\K \E)$:
\begin{equation}
\begin{aligned}
    &\underset{\rempc, \emax}{\text{max.}} ~ && 3 \log (\rempc) + \sum_{l=1}^3 \log(2 \zeta_l^T K\emax)\\
    &~~~\text{s.t.}
          && \left\Vert a_i \mat{
            \rempc I_{3} & 0_{3 \times 3} \\
            0_{3 \times 3} & 0_{3 \times 3}
          } \right\Vert_2 
          + a_i \mat{0_{3 \times 1} \\ \K \epsilon_j}  \\
          & && \qquad \leq b_i - m \utrajmax - \fblinmax, i\in \mathcal{I}, j=1,\ldots,8
\end{aligned}
\label{eq:term_set_vol_max}
\end{equation}
where $\epsilon_j \in \R^3$ are the vertices of $\E$ and $\zeta_l\in \R^3$ is the $l$-th standard basis vector.
Note that this bound exists from \cref{as:zero_in_interior} and $\fblinmax \rightarrow 0$ for $\emax \rightarrow 0$ (due to $\fe \rightarrow 0$ for $\eom \rightarrow 0.$ and \cref{eq:term_set_femax}).
The constraints ensure that \cite[chap. 8.4.2]{Boyd2004}:
\begin{align}
    \Uempcupper \times (-\K \E) \subseteq \M \U \ominus \Utraj \ominus \M \Ufl.
\label{eq:constraints_term_set_vol_max}
\end{align}

Based on the above, we can deduce the forward invariance of $\Xterm$ and the asymptotic stability of \eqref{eq:err_dyn_orbiting} under the terminal controller in \cref{eq:term_contr}.
\CDCversion{The proofs of all statements below can be found in \cite{stockner2025}.}
\arxivVersion{The proofs of all statements below can be found in \cref{sec:appendix:proofs}.}
\begin{lemma} \label{lem:stability_term_contr}
Consider the orbit error dynamics \cref{eq:err_dyn_orbiting} under the terminal control law  \cref{eq:term_contr} and the set $\Xterm$ defined in \cref{eq:terminal_set} with parameters chosen based on \eqref{eq:term_set_vol_max}. Let further \cref{as_reftraj,as:traj_dynfeas} hold. Then, \cref{eq:term_cond} holds. In addition, $e_k \rightarrow 0$ as $k\rightarrow +\infty.$ 
\end{lemma}

\subsection{Terminal cost} \label{sec:analysis:term_cost}
Given the terminal set $\mathcal{T}$ obtained in \Cref{sec:analysis:term_set}, we will next design the terminal cost $l_{\mathcal{T}}:\R^9\rightarrow \R_{\geq 0}.$ First, we make the following assumption on the weight matrix $Q_e$ in \eqref{eq:mpc}:
\begin{assumption} \label{as:QQ_matrix}
The matrix $Q_e\in \R^{9\times 9}$ is a positive definite matrix of the form $\Qe := \begin{bmatrix}
        \Qeempc & 0_{6 \times 3} \\ 0_{3 \times 6} & Q_{e,\omega}
    \end{bmatrix}, $
where $\Qeempc\in \R^{6\times 6}$ and $Q_{e,\omega} \in \R^{3\times 3}.$
\end{assumption}
Given $Q_u$ in \eqref{eq:mpc} let $\Qutilde := (\Minv)^T \Qu \Minv,$ where $\M$ is defined in \eqref{eq:def_M}. For the remainder of this section, let:
$\ufblin := -\feidx{k} + \mat{0_{3 \times 1}^T & -  \eomidx{k}^T K^T}^T $
and
$\utermCeMPC := \mat{ \controlEmpc_k^T \Rot^T  & 0_{1 \times 3}}^T.$
Before we formally introduce the cost function, we provide two technical lemmas:
\begin{lemma} \label{lem2}
Consider the system \cref{eq:err_dyn_orbiting} under the terminal control law \cref{eq:term_contr} and assume $e_s \in \Xterm$ for some $s \in \N$. Then, $\sum_{k=s}^{\infty} (\ufblin)^T \Qutilde \ufblin \leq \Theta_1(e_s) + \sum_{k=s}^{\infty} \eomidx{k}^T (2 \Vert \Qutilde \Vert_2 \K^T \K) \eomidx{k},$
where $\Qutilde \succ 0$,
and $\Theta_1: \R^3 \rightarrow \R$ is a function satisfying $\Theta_1(e_0) \leq \alpha_{2,1}(\Vert e_0 \Vert),$ where $\alpha_{2,1}: \R_{\geq 0} \rightarrow \R_{\geq 0}$ is a \classkappainf function. 
\end{lemma}

\begin{lemma} \label{lem3}
Consider system \cref{eq:err_dyn_orbiting} under the terminal control law \cref{eq:term_contr} and assume $e_s \in \Xterm$ for some $s \in \N$. Then, 
    $\sum_{k=s}^{\infty} (\ufblin)^T \Qutilde \utermCeMPC \leq \Theta_2(\eomidx{s})$
where $\Qutilde$ is the same as in \cref{lem2}, $\Theta_2: \R^3 \rightarrow \R$ satisfies
$\Theta_2(e_0) \leq \alpha_{2,2}(\Vert e_0 \Vert)$ and $\alpha_{2,2}: \R_{\geq 0}\rightarrow \R_{\geq 0}$ is a \classkappainf function. 
\end{lemma}

\newcommand{\Qet}{\Qeempc}
\newcommand{\Qer}{Q_{e,\omega}}
\newcommand{\Qut}{\bar{Q}_{u,p}}

Next, let $\Theta_3(e_s) := \eomidx{s}^T P \eomidx{s},$ for $s \in \N$
where $P \in \R^{3 \times 3}$ is the solution to the Lyapunov equation for an autonomous system with cost matrix $\Qer + 2 \Vert \Qutilde \Vert_2 \K^T \K \succ 0$ and linear dynamics \cref{eq:omnew_ter}. 

Further, using \eqref{eq:term_contr}, $\Theta_1(\cdot),\Theta_2(\cdot)$ defined in Lemmas 2 and 3, $\Theta_3(\cdot)$ and the parametric cost $\termC(\cdot)$ obtained when solving the eMPC in \eqref{eq:empc_contr}, we define $l_{\mathcal{T}}(e_s), s\in \mathbb{N}$ as:
\begin{equation}
\termC(e_s) := \costEmpc(e_{p, s}) + \Theta_1(e_s) + \Theta_2(e_s)  + \Theta_3(e_s).     \label{eq:terminal_cost}  
\end{equation}
Based on the above we can deduce the following:
\begin{proposition}\label{prop:term_cost}
Consider the system \cref{eq:err_dyn_orbiting} under the terminal control law \cref{eq:term_contr} and assume that $e_s\in \Xterm$ for some $s\in \mathbb{N}.$ Assume further that Asms. \ref{as_reftraj},\ref{as:traj_dynfeas},\ref{as:QQ_matrix} hold. Let further $l_{\mathcal{T}}(e_s)$ be defined as in \eqref{eq:terminal_cost}. Then, 
$$\sum_{k=s}^{\infty} e_k^T \Qe e_k + (\contrTerm_k)^T \Qu \contrTerm_k \leq l_{\mathcal{T}}(e_s).$$
\end{proposition}

\arxivVersion{
    \begin{figure}[tpb]
        \centering   \includegraphics[width=0.9\linewidth,trim={0 4cm 0 8cm},clip]{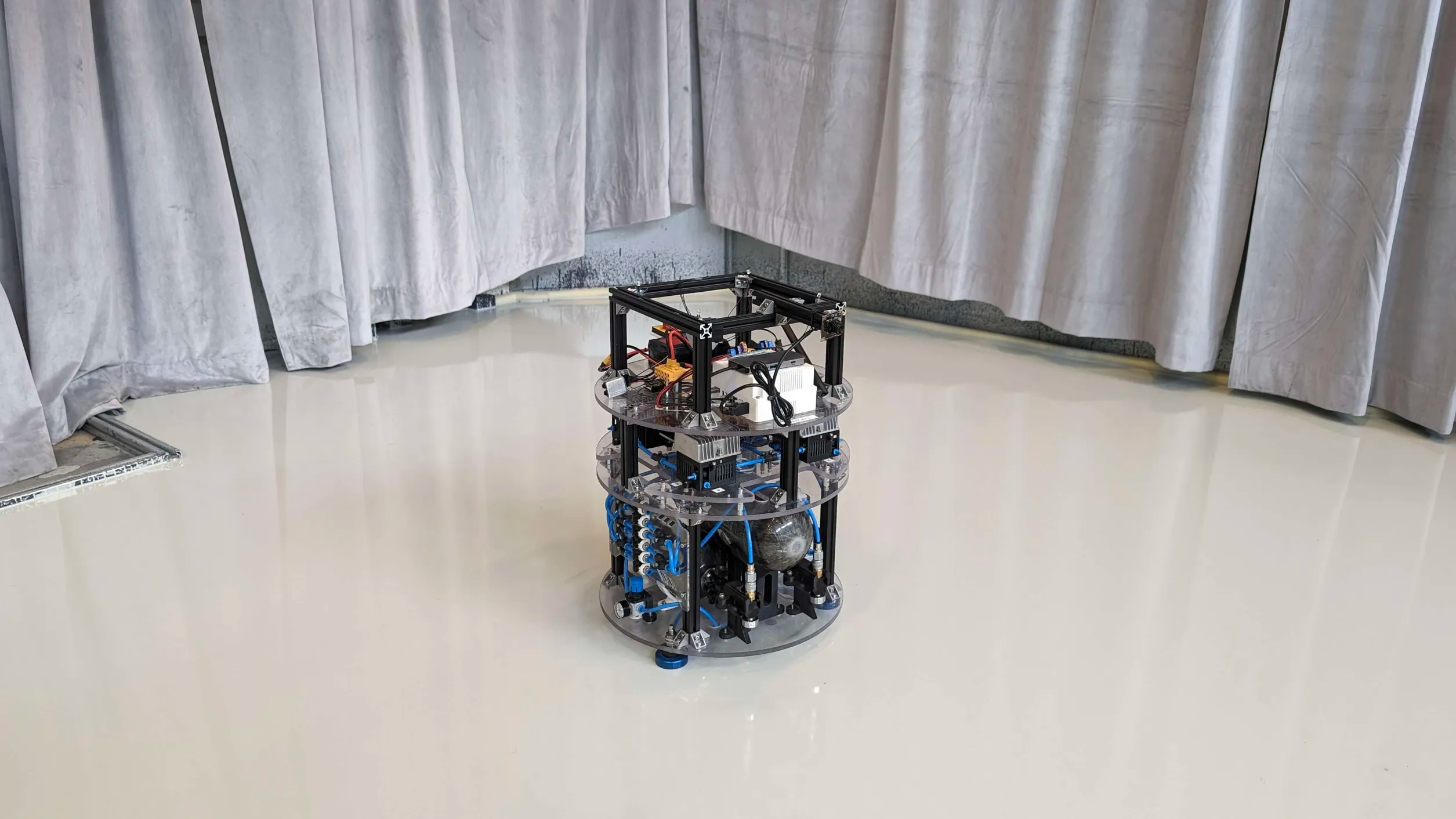}
        \caption{The ATMOS free-flyer available at the KTH Space Robotics Laboratory \cite{roque2025opensourcemodularspacesystems}.}
        \label{fig:freeflyer_picture}
        \vspace{-2em}
    \end{figure}
}

\subsection{Recursive Feasibility and Stability Analysis} \label{sec:analysis:stability}
This section examines the stability of \eqref{eq:err_dyn_orbiting} and \eqref{eq:dynamics_continuous}.
\begin{theorem} \label{prop:stability_orbit}
    Consider the system  \cref{eq:err_dyn_orbiting} and let \cref{as:zero_in_interior,as_reftraj,as:traj_dynfeas} hold. Consider further the MPC problem  \eqref{eq:mpc} with a terminal set and cost defined in \eqref{eq:terminal_set} and \eqref{eq:terminal_cost}, respectively, and assume that \eqref{eq:mpc} is feasible at $k=0$.
    Then, the closed-loop system $e_{k+1} = \phi(e_k, \contrMPC_k)$ with $\contrMPC_k$ from \eqref{eq:u_mpc} is asymptotically stable.
\end{theorem}

\begin{theorem} \label{theo:main}
Consider the spacecraft dynamics \eqref{eq:dynamics_continuous} and let \cref{as:zero_in_interior,as_reftraj,as:traj_dynfeas,as:discretization_error} hold. Let $x:\R_{\geq 0} \rightarrow \R^{16} $ be the solution of 
\eqref{eq:dynamics_continuous} under the MPC control law $\varrho:\R_{\geq 0} \rightarrow \Ures$ defined as
$
\varrho(t)=\contrMPC_k, \quad \forall t\in [t_k, t_{k+1}), k\in \N,
$
where $\contrMPC_k, k\in \N$ is defined in \eqref{eq:u_mpc} and $t_k, k\in \N$ as in \cref{as_reftraj}. Then, as $t\rightarrow\infty$, 
$
\Vert p_r(t)-\traj_{[1:3]}(t)\Vert \leq r_s,
$
where $r_s \geq r+\beta,$ $r$ is the radius of the orbit designed in \cref{sec:control} and $\beta$ is the same as in \cref{as:discretization_error}.
\end{theorem}


\section{Numerical Results}
\label{sec:application}
This section presents results for the developed controller implemented in 2D on the ATMOS platform \cite{roque2025opensourcemodularspacesystems} and simulated in 3D\footnote{
Both code and a small video of the experiment are available in \url{https://github.com/DISCOWER/fault-tolerant-mpc}. The inertial parameters are available in the file \href{https://github.com/DISCOWER/fault-tolerant-mpc/blob/main/data/InertialProperties.md}{\texttt{data/InertialParameters.md}}.
}. 
Both controllers used Python with CasADi \cite{Andersson2019} and \texttt{pympc} \cite{pympc} with a sampling time $\dt = \qty{0.1}{\second}$ and ran on an HP ENVY x360 laptop (AMD Ryzen™ 7 5700U CPU). ROS2 was used as interface to the ATMOS platform, and Gazebo for simulation.

\paragraph*{2D implementation}
The 2D implementation was done with a simplified model and controller where all variables outside of the $x$-$y$-plane and their respective inputs are zero, representing the ATMOS free-flyer \cite{roque2025opensourcemodularspacesystems}.
The robot has eight thrusters, composing $D$ according to 
\begin{align*}
    D = \left[ \begin{smallmatrix}
        1& -1&  1& -1&  0&  0&  0&  0\\
        0&  0&  0&  0&  1& -1&  1& -1\\
        d& -d& -d&  d&  d& -d& -d&  d
    \end{smallmatrix} \right],
\end{align*}
where $d=\qty{12}{\cm}$. 
If three or less actuators fail, \cref{as:zero_in_interior} is satisfied: as four actuators produce positive and four negative torque (see the third row of $D$), three functional ones can compensate for three failed ones. One thruster pair remains, which can produce both positive and negative torque. Note that \cref{as:zero_in_interior} does not require $0 \in \Uresidx{1:2}$, i.e., a resulting force after the compensation is allowed.

The experiment video, available in the codebase repository, shows fault-free trajectory tracking of a circular path until second $12$, when three actuators fail with $f_{1} = f_{3} = f_{5} = \qty{0}{\newton}$. The system recovers by adaptively switching to the proposed controller and tracking an orbit around the trajectory. We choose $\fvirt = \qty{1.98}{\newton}$ (the Chebyshev center of $\Ures$) and $\omegades = \qty{0.5}{\radian\per\second}$, giving $r = \qty{0.47}{\metre}$.
The control horizon is $N = 20$, the tuning $\Qe = I_{5 \times 5}$ and $\Qu = 0.1 I_{3 \times 3}$, $\K = 0.1$ and eMPC horizon $\hat{N}=50$. The controller shows good robustness to parameter errors as the thrusters are fueled through gas tanks and the flushing out gas changes $m$, $J$ and $\fmax$ over time.
\CDCversion{The terminal set resulting from our approach is evaluated in \cite{stockner2025}.}

\arxivVersion{
    \Cref{fig:comparison_term_sets} shows the terminal sets calculated from two different terminal controllers for \cref{eq:mpc}, taken from \cite{raphael}: the blue one corresponds to the \cref{eq:term_contr}, while the red one replaces the eMPC part in \cref{eq:term_contr} with an LQR controller. As the eMPC controller handles actuator saturations, the resulting terminal set is larger.

    \begin{figure}[t!]
    \begin{center}
    \includegraphics[width=0.6\columnwidth]{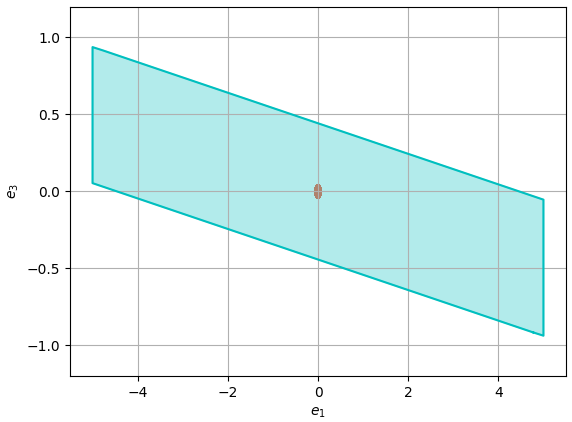}
        \caption{Slice through $\Xterm$ for $e_2=e_3=e_5=0$. Red: Linear controller. Cyan: eMPC controller. Taken from \cite{raphael}.}
        \label{fig:comparison_term_sets}
    \end{center}
    \end{figure}
}

\paragraph*{3D simulation}
For the 3D scenario, a spacecraft was simulated.
It suffers $f_{11} = f_{12} = \fmax$, i.e. maximal acceleration into the positive $y$-direction. The system initializes with failures at the state $p_r = \mat{1&0&1}^T$\unit{\metre}, $v_r = \mat{1&0&0.5}^T$\unit{\metre\per\second}, $q = \mat{0.033&0.27&0.39&0.88}^T$ and $\omega = \mat{0.3&0.8&-0.1}^T$\unit{\radian\per\second} tracking a constant trajectory at the origin (i.e., a fixed setpoint).
The parameters are $N = 15$, $\Qe = \diag(1, 1, 1, 1, 1, 1, 2, 2, 2)$, $\Qu = \diag(0.1, 0.1, 0.1, 0.01, 0.01, 0.01)$ and in the terminal controller $\K = I_{3 \times 3}$ and eMPC horizon $\hat{N} = 15$. The spiral parameters are $\omegades = \qty{0.6}{\radian\per\second}$ and $\fvirt = \qty{3.5}{\newton}$ so $r$ follows as \qty{0.58}{\metre}. The resulting spacecraft and orbit center paths are shown in \cref{fig:three_d_path} and individual states and inputs in \cref{fig:states_inputs}. The system under failures is quickly stabilized to the setpoint.
\begin{figure}
\begin{center}
  \begin{tikzpicture}
    \tikzstyle{pathlinestyle}=[
          mark=|,
          mark repeat=5
        ]
    \begin{axis}[
          view={210}{30},
          xlabel=$x$,
          ylabel=$y$,
          zlabel=$z$,
          z label style={rotate=90},
          label shift=-4mm,
          grid=major,
          height=5cm,
          legend style={
            font=\footnotesize,
            at={(0.52,0.98)},
            anchor=north west,
          },
        ]
        \addplot3+ [pathlinestyle, mark=x] table[x=position_x,y=position_y,z=position_z,col sep=semicolon] {debug_data.csv};
        \addplot3+ [pathlinestyle] table[x=circle_position_x,y=circle_position_y,z=circle_position_z,col sep=semicolon] {debug_data.csv};
        \legend{Spacecraft, Orbit center}
    \end{axis}
  \end{tikzpicture}
\caption{Spacecraft path and orbit center. Plotted is the position in meter at different time steps, with one mark every \qty{0.5}{\second}.}
\label{fig:three_d_path}
\vspace{-1em}
\end{center}
\end{figure}
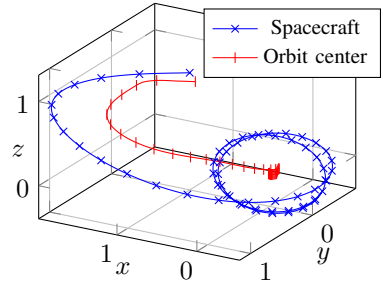

\newcommand{\myfontsize}{\normalfont}
\newcommand{\myfontweight}{\fontseries{l}\selectfont} 

\tikzstyle{mylinestyle}=[
          mark=none,
          line width=1.2pt 
        ]

\pgfplotsset{
    tick style={draw=none},
    axis line style={line width=0.15pt},
    myplotstyle/.style={
        baseline,
        width=0.32\linewidth, 
        height=22mm,
        grid=major, 
        grid style={gray!30},
        legend style={at={(0.98,0.98)},anchor=north east,font=\tiny,row sep=-3pt},
        scale only axis,
        ylabel shift=-2mm,
        tick label style={font=\footnotesize, text=black}, 
        label style={font=\footnotesize, text=black},
    },
    myplotstyleright/.style={
      myplotstyle,
      yticklabel pos=right,
      ylabel near ticks,
      label style={font=\footnotesize, text=black},
    }
}
\renewcommand{\arraystretch}{-1.1} 

\begin{figure}
\begin{center}
\begin{tabular}{r@{\hspace{2pt}}l}
    \begin{tikzpicture}
    \pgfplotsset{small} 
      \begin{axis}[
          myplotstyle,
          ylabel=Position Error,
          y unit=\si{\meter},
        ]
        \addplot[mylinestyle,color=orange]           table[x=time,y=circle_position_x,col sep=semicolon] {debug_data.csv}; 
        \addplot[mylinestyle,color=blue]            table[x=time,y=circle_position_y,col sep=semicolon] {debug_data.csv}; 
        \addplot[mylinestyle,color=magenta] table[x=time,y=circle_position_z,col sep=semicolon] {debug_data.csv}; 
        \legend {X, Y, Z}
      \end{axis}
    \end{tikzpicture}
  &
    \begin{tikzpicture}
    \pgfplotsset{small} 
      \begin{axis}[
          myplotstyleright,
          ylabel=Velocity Error,
          y unit=\si{\meter\per\second},
        ]
        \addplot[mylinestyle,color=orange]            table[x=time,y=circle_velocity_x,col sep=semicolon] {debug_data.csv}; 
        \addplot[mylinestyle,color=blue]             table[x=time,y=circle_velocity_y,col sep=semicolon] {debug_data.csv}; 
        \addplot[mylinestyle,color=magenta]  table[x=time,y=circle_velocity_z,col sep=semicolon] {debug_data.csv}; 
        \legend {X, Y, Z}
      \end{axis}
    \end{tikzpicture}
  \\
    \begin{tikzpicture}
      \begin{axis}[
          myplotstyle,
          xlabel=time, 
          ylabel=Input Force,
          x unit=\si{\second}, 
          y unit=\si{\newton},
          ymin=-7,
          ymax=7,
        ]
        \addplot[mylinestyle,color=orange]           table[x=time,y=force_x,col sep=semicolon] {debug_data.csv};
        \addplot[mylinestyle,color=blue]            table[x=time,y=force_y,col sep=semicolon] {debug_data.csv}; 
        \addplot[mylinestyle,color=magenta] table[x=time,y=force_z,col sep=semicolon] {debug_data.csv}; 
        \legend {X, Y, Z}
      \end{axis}
    \end{tikzpicture}
  & 
    \begin{tikzpicture}
      \begin{axis}[
          myplotstyleright,
          xlabel=time, 
          ylabel=Input Torque,
          x unit=\si{\second}, 
          y unit=\si{\newton\meter},
        ]
        \addplot[mylinestyle,color=orange]           table[x=time,y=torque_x,col sep=semicolon] {debug_data.csv}; 
        \addplot[mylinestyle,color=blue]            table[x=time,y=torque_y,col sep=semicolon] {debug_data.csv}; 
        \addplot[mylinestyle,color=magenta] table[x=time,y=torque_z,col sep=semicolon] {debug_data.csv}; 
        \legend {X, Y, Z}
      \end{axis}
    \end{tikzpicture}
\end{tabular}
\caption{Errors (in global coordinates) and inputs (in body coordinates) of the 3D spacecraft.}
\vspace{-2em}
\label{fig:states_inputs}
\end{center}
\end{figure}
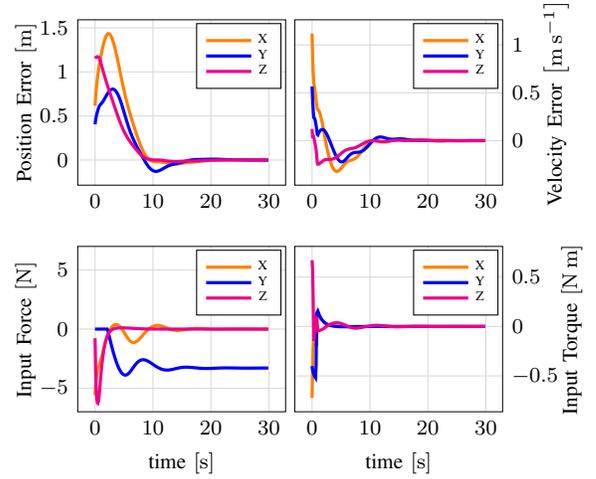

\section{Conclusions and Future Work}
\label{sec:summary}
We propose an MPC scheme for spacecraft control under actuator failures that exploits the spacecraft's dynamics on a local orbit for partial error compensation. A novel feedback-linearization and eMPC-based terminal controller achieves larger terminal sets than existing approaches and enables orbit center trajectory tracking for safe deorbiting or navigation to repair hubs despite failures. Future work will extend to robust MPC and more complex system dynamics.

\printbibliography

\section{Appendix: Proofs}
\label{sec:appendix:proofs}
\begin{proof}[Proof of Lemma \ref{lem:stability_term_contr}]
    \Cref{eq:term_cond2} follows directly by construction from \cref{eq:constraints_term_set_vol_max}, \cref{eq:term_contr} and the definition of the Minkowski sum.
    Next, we show the validity of \cref{eq:term_cond1}. Given $e_k \in \Xterm$, \cref{eq:term_cond2} guarantees that $\contrTerm_k \in \U$. At step $k$, the dynamics of $e_p$ and $\eom$ thus simplify to the linear systems in \eqref{eq: closedloop_term}.
    By choice of $\K$, it follows that $|e_{\omega, k+1}| \leq |e_{\omega, k}|$ for each component. Therefore, by convexity of $\E$ and since $0 \in \Interior{\E},$ it follows that $\eomidx{k} \in \E \Rightarrow \eomidx{k+1} \in \E$.
   Under appropriately chosen terminal ingredients the eMPC  \eqref{eq:empc_contr} can be shown to be recursively feasible. Then, considering the definition of $\XempcFeasible,$ we can conclude that $e_{p,k} \in \XempcFeasible \Rightarrow e_{p,k+1} \in \XempcFeasible$ and \cref{eq:term_cond1} follows.
     Finally, since $\controlEmpc_k$ asymptotically stabilizes \cref{eq:empc_sys}, and $\K$ is chosen to make \cref{eq:omnew_ter} asymptotically stable, we have $e_k \rightarrow 0$ as $k \rightarrow +\infty.$ This concludes the proof.
\end{proof}


\begin{proof}[Proof of Lemma \ref{lem2}]
Let $L$ be the matrix obtained by the Cholesky decomposition $\Qutilde,$ i.e., it satisfies $\Qutilde = L L^T.$ Then,
\begin{align*}
        &(\ufblin)^T \Qutilde \ufblin  =  \left\Vert L^T \ufblin \right\Vert_2^2 \leq \Vert L^T \Vert_2^2 \left\Vert \ufblin \right\Vert_2^2 \\
        \leq &2 \Vert \Qutilde \Vert_2 \left( \Vert \feidx{k} \Vert_2^2 + \left\Vert \mat{0 \\ \K \eomidx{k}} \right\Vert_2^2 \right) \\
        = &2 \Vert \Qutilde \Vert_2 \Vert \feidx{k} \Vert_2^2 + \eomidx{k}^T (2 \Vert \Qutilde \Vert_2 \K^T \K) \eomidx{k},
    \end{align*}
where in the second inequality the Arithmetic-Mean-Quadratic-Mean inequality is used.
Since $\feidx{k}$ is a polynomial, $\Vert \feidx{k}) \Vert_2^2$ is also a polynomial and thus can be expressed using suitable constants $\varsigma_{i,j,l} \in \R$.
    This enables calculating the infinite-horizon sum of the term using \eqref{eq:omnew_ter}. Here, we use $\iota_1=i, \iota_2=j, \iota_3=l$ and index the $m$-th vector row of $\eomidx{k}$ as $\eomidx{k,[m]}$:
    \begin{align*}
        &\sum_{k=s}^{\infty}\Vert \feidx{k} \Vert_2^2  =: \frac{1}{2 \Vert \Qutilde \Vert_2} \sum_{k=s}^{\infty} \Lambda_1(e_k) \\
        = & \sum_{k=s}^{\infty}\sum_{0\leq i,j,l \leq 4} \varsigma_{i, j, l} \prod_{m=1}^3 \eomidx{k,[m]}^{\iota_m} \\
        = &\sum_{k=s}^{\infty}\sum_{0\leq i,j,l \leq 4} \varsigma_{i, j, l} \prod_{m=1}^3 \eomidx{s,[m]}^{\iota_m} [(1-\dt \ki{m})^{\iota_m}]^k \\
        = & \sum_{0\leq i,j,l \leq 4} \varsigma_{i, j, l} \frac{1}{1 - \prod_{m=1}^3(1-\dt \ki{m})^{\iota_m}} \prod_{m=1}^3 \eomidx{s,[m]}^{\iota_m} \\
        =: & \frac{1}{2 \Vert \Qutilde \Vert_2} \Theta_1(e_s).
    \end{align*}
    In the second to last step, the geometric series has been applied and due to the choice of $\K$ it holds $|1-\dt \ki{m}|< 1$. 
    Note that \cref{lem:stability_term_contr} ensures the feedback linearization from \cref{eq:err_dyn_orbiting} to \cref{eq:omnew_ter} remain valid for all $k \in \N$.
    
    $\Theta_1(\cdot)$ satisfies $\Theta_1(0)=0$ and $\Theta_1(\Vert e_0 \Vert)$ strictly increasing with $\Theta_1(\Vert e_0 \Vert)\rightarrow \infty$ for $e_0 \rightarrow \infty$. The last statement follows by equality with a sum of 2-norms. Thus, $\Theta_{1}(e_0) \leq \alpha_{2,1}(\Vert e_0 \Vert)$.
\end{proof}


\begin{proof}[Proof of Lemma \ref{lem3}]
    First, consider
    \begin{align*}
        &(\ufblin)^T \Qutilde \utermCeMPC
        \leq \left\Vert (\ufblin)^T \Qutilde \right\Vert_2  \left\Vert \utermCeMPC \right\Vert_2 
        \leq \controlEmpc_{\max} \left\Vert (\ufblin)^T \Qutilde \right\Vert_2 
    \end{align*}
    with the maximal control input of the eMPC controller $\controlEmpc_{\max}$.
    
    An approach similar to the proof of \cref{lem2} can be applied introducing the polynomial coefficients $\rho_{i,j,l} \in \R$ 
    \begin{align*}
        &\sum_{k=s}^{\infty}\left\Vert (\ufblin)^T \Qutilde \right\Vert_2 \\
        \leq& \sum_{k=s}^{\infty} \sum_{0\leq i,j,l \leq 4} \sqrt{ \left\vert \rho_{i, j, l} \prod_{m=1}^3 \eomidx{k,[m]}^{\iota_m} \right\vert} 
        =:\frac{1}{2 \controlEmpc_{\max}} \sum_{k=s}^{\infty} \Lambda_2(e_k) \\
        = &\sum_{k=s}^{\infty} \sum_{0\leq i,j,l \leq 4} \sqrt{ 
        \left\vert \rho_{i, j, l} \prod_{m=1}^3 \eomidx{s,[m]}^{\iota_m} \right\vert} \left[ \prod_{m=1}^3(1 - \dt \ki{m})^{\frac{\iota_m}{2}} \right]^k \\
        = &\sum_{0\leq i,j,l \leq 4} \frac{1}{1 - \prod_{m=1}^3(1-\dt \ki{m})^{\frac{\iota_m}{2}}} \sqrt{ \left\vert \rho_{i, j, l} \prod_{m=1}^3 \eomidx{s,[m]}^{\iota_m} \right\vert} \\
        := & \frac{1}{2 \controlEmpc_{\max}} \Theta_2(e_{s})
   \end{align*}
    with a similar argument as for \cref{lem2}. 
\end{proof}
Before we proceed with the proof of Proposition \ref{prop:stability_orbit}, we will first introduce the following technical lemma:
\begin{lemma} \label{lem1}
Consider matrices $A, B, C \in \R^{n \times n}$ with $\min (\lambda(A)) \geq \max( (\lambda (C))$, $A$, $C$ symmetric and $B$ orthogonal. Then, $A - B^{-1} C B \succeq 0.$ 
\end{lemma}
\begin{proof}
First, notice that $(B^{-1} C B)^T = B^T C^T (B^{-1})^T = B^{-1} C B$, i.e. it is symmetric due to $C$ symmetric and $B$ orthogonal. For any vector $x\in \R^n$ we have:
\begin{align*}
x^T [A - B^{-1} C B] x 
&= x^TAx - x^T [B^{-1}CB]x \\
        &\geq x^T x [\min (\lambda (A)) - \max(\lambda (B^{-1}CB))] \\
        &= x^T x [\min (\lambda(A)) - \max(\lambda (C))] \geq 0.
    \end{align*}
where the last equality follows by  $\lambda(B^{-1} C B) = \lambda(C)$ \cite{hungerford2012algebra} and $\min (\lambda (A)) \geq \max( \lambda (C))$. Thus, $A - B^{-1} C B \succeq 0.$

\end{proof}
\begin{proof}[Proof of Proposition \ref{prop:term_cost}]
Given $\bar{Q}_u$ choose $\Quhat$ such that $ \Quhat - \Rot^T \bar{Q}_{\mathrm{u}, [1:3, 1:3]} \Rot \succeq 0,$ where $R$ is the rotation matrix of \eqref{eq:dynamics_continuous}.
This matrix is guaranteed to exist by \cref{lem1}.
Next, observe $(\utermCeMPC)^T \Qutilde \utermCeMPC = \controlEmpc_k^T \Rot^T \bar{Q}_{\mathrm{u}, [1:3, 1:3]} \Rot \controlEmpc_k \leq \controlEmpc_k^T \Quhat \controlEmpc_k$. As a result, the following holds:
\begin{equation}
\sum_{k=s}^{\infty} \left\{ e_{p,k}^T \Qet e_{p,k} + \controlEmpc_k^T \Quhat \controlEmpc_k \right\} \leq \costEmpc(e_{p, s}), \label{eq:upper1}
\end{equation}
where $\Qeempc$ is the submatrix of $Q_e$ in \cref{as:QQ_matrix}.  

Then, using \eqref{eq:upper1}, Lemmas \cref{lem2,lem3} and the properties of $\Theta_3(\cdot)$ we have:
\begin{equation*}
\begin{aligned}
     & \sum_{k=s}^{\infty} e_k^T \Qe e_k + (\contrTerm_k)^T \Qu \contrTerm_k \\ 
    =& \sum_{k=s}^{\infty} e_k^T \Qe e_k + (\ufblin + \utermCeMPC)^T (\Minv)^T \Qu \Minv (\ufblin + \utermCeMPC) \\
    =& \sum_{k=s}^{\infty} \left\{ e_k^T \Qe e_k +    (\utermCeMPC)^T \Qutilde \utermCeMPC + 2 (\ufblin)^T \Qutilde \utermCeMPC \right. \\
    & \qquad \left. + (\ufblin)^T \Qutilde \ufblin  \right\} \\
    \leq& \infsum \left\{ e_{p,k}^T \Qet e_{p,k} + \controlEmpc_k^T \Quhat \controlEmpc_k \right\} + \infsum \left\{ \eomidx{k}^T \Qer \eomidx{k} \right. \\
    & \qquad \left. +  2 (\ufblin)^T \Qutilde \utermCeMPC + (\ufblin)^T \Qutilde \ufblin \right\} 
    \\
    \leq& \sum_{k=s}^{\infty} \left\{ e_{p,k}^T \Qet e_{p,k} + \controlEmpc_k^T \Quhat \controlEmpc_k + \Lambda_1(e_k) + \Lambda_2(e_k) \right. \\ 
    & \qquad \left. + \eomidx{k}^T (\Qer + 2 \Vert \Qutilde \Vert_2 \K^T \K) \eomidx{k}
    \right\} \\
    =& \termC(e_s),
\end{aligned}
\end{equation*}
where in the last inequality we used the fact that $\Qer \succ 0$, $\K^T \K \succeq 0$ and where $\Lambda_1(\cdot), \Lambda_2(\cdot)$ are defined as in the proofs of \cref{lem2,lem3}.  
\end{proof}
\begin{proof}[Proof of Theorem \ref{prop:stability_orbit}]
In the following we will show the recursive feasibility of \eqref{eq:mpc} and the asymptotic stability of the closed loop system.

The \textit{recursive feasibility} property of \eqref{eq:mpc} follows from \Cref{lem:stability_term_contr} using standard arguments.
Furthermore, using $\termC(\cdot)$ in \eqref{eq:terminal_cost}, notice that $\termC(e_N)-\termC(e_{N+1})=e_{p,N}^T \Qet e_{p,N} + \controlEmpc_N^T \Quhat \controlEmpc_N + \Lambda_1(e_N) + \Lambda_2(e_N)  + \eomidx{N}^T (\Qer + 2 \Vert \Qutilde \Vert_2 \K^T \K) \eomidx{N},$ which by means of \cref{eq:terminal_cost} leads to $\termC(e_{N+1})-\termC(e_N) \leq - e_N^T \Qe e_N - \contrTerm(e_N,N)^T \Qu \contrTerm(e_N,N).$
    Since the stage cost is quadratic, it satisfies $e_k^T \Qe e_k + u_k^T \Qu u_k \geq \alpha_1(\Vert e_k \Vert), \forall k \in \N$ for a \classkappainf function $\alpha_1$.
    In addition, by \cref{lem2,lem3} and definition of $\Theta_3(\cdot),$ it follows that $\Theta_i(\cdot), i=1, 2,3 $ are all are upper bounded by a \classkappainf function and the same statement follows for $\costEmpc(\cdot)$ directly by design. Therefore, $\termC(e_N) \leq \alpha_2(\Vert e_N \Vert)$ with the \classkappainf function $\alpha_2$. 
    Furthermore, $\alpha_3(\Vert e_N \Vert) \leq \termC(e_N)$ by \cref{eq:terminal_cost}. Based on the aforementioned analysis and the properties of the terminal ingredients in \Cref{lem:stability_term_contr}, the asymptotic stability of $e_{k+1} = \phi(e_k, \contrMPC_k)$ follows from \cite[Thm. 5.13, Rem. 5.17]{grune2017nonlinear}. 
\end{proof}


\begin{proof}[Proof of Theorem \ref{theo:main}]
For $t\in [t_k,t_{k+1})$ we have:
$
\Vert p_r(t)-\traj_{[1:3]}(t)\Vert \leq \Vert p_r(t)-c(t)\Vert + \Vert c(t)-c_k \Vert + \Vert c_k-\traj_{[1:3]}(t) \Vert. 
$
By \cref{eq:center_p}, it follows that $\Vert p_r(t)-c(t)\Vert \leq r.$ In addition, due to \cref{as:discretization_error} it holds that $\Vert c(t)-c_k \Vert \leq \beta$ and due to \cref{as_reftraj}, $\traj_{[1:3]}(t)=\traj_{k,[1:3]}, t\in[t_k,t_{k+1}).$ Due to the stabilizing properties of \cref{eq:mpc} stated in \cref{prop:stability_orbit}, $\| c_{k+1} - \traj_{k+1,[1:3]}\| < \| c_k - \traj_{k,[1:3]}\|$. Solving \cref{eq:mpc} at the time step $t_{k+2}$ yields $\| c_{k+2} - \traj_{k+2,[1:3]}\| < \| c_{k+1} - \traj_{k+1,[1:3]}\| < \| c_k - \traj_{k,[1:3]}\|$. Note that at $t_{k+1}$ a switching of the control input occurs from $\Fuidx{k}$ to $\Fuidx{k+1}$. However, on
the interval $[t_k,t_{k+1})$, this switch occurs only once and thus $\Fuidx{k}$ is applied for almost all $t \in [t_k,t_{k+1}]$. Therefore, since from \cref{prop:stability_orbit},  $ \| c_k - \traj_{k,[1:3]} \| \rightarrow 0$ as $k\rightarrow\infty,$ it follows that $\Vert p_r(t)-\traj_{[1:3]}(t)\Vert \leq r_s,$ as $t \rightarrow \infty.$
\end{proof}

\section{Appendix: Simulation parameters}

\subsection{Inertial Properties of 2D Freeflyer and 3D Spacecraft Examples}

\begin{table}[h!]
    \centering
    \begin{tabular}{ccc}
    \toprule
         &  2D Freeflyer& 3D Spacecraft
         \\ \midrule
         $\delta$ [\unit{\second}]&  $0.1$& $0.1$\\
         Mass $m$ [\unit{\kilogram}]&  14.5& 16.8\\
         Inertia $J$ [\unit{\kilogram\metre\squared}]&  $0.370$& $\diag(0.2, 0.3, 0.25)$\\ Maximum force $f^{\mathrm{max}}$ [\unit{\newton}]& 1.75&1.75\\
         \bottomrule
    \end{tabular}
    \caption{Spacecraft parameters used for controller and simulation}
    \label{tab:placeholder}
\end{table}

\subsection{Allocation Matrices}
\noindent
2D ATMOS platform:
\begin{align*}
    D = 
    \left[ \begin{smallmatrix}
        1& -1&  1& -1&  0&  0&  0&  0\\
        0&  0&  0&  0&  1& -1&  1& -1\\
        d& -d& -d&  d&  d& -d& -d&  d
    \end{smallmatrix} \right]
\end{align*}
with $d=\qty{12}{\centi\meter}$.

\noindent
3D Spacecraft:
\begin{align*}
     D = 
    \left[ \begin{smallmatrix}
-1 & -1 & 1 & 1 & -1 & -1 & 1 & 1 & 0 & 0 & 0 & 0 & 0 & 0 & 0 & 0 \\ 
0 & 0 & 0 & 0 & 0 & 0 & 0 & 0 & -1 & -1 & 1 & 1 & 0 & 0 & 0 & 0 \\ 
0 & 0 & 0 & 0 & 0 & 0 & 0 & 0 & 0 & 0 & 0 & 0 & -1 & 1 & -1 & 1 \\ 
0 & 0 & 0 & 0 & 0 & 0 & 0 & 0 & 0 & 0 & 0 & 0 & -a & a & a & -a \\ 
-c & c & c & -c & -c & c & c & -c & 0 & 0 & 0 & 0 & 0 & 0 & 0 & 0\\ 
a & a & -a & -a & -a & -a & a & a & -b & b & b & -b & 0 & 0 & 0 & 0
     \end{smallmatrix} \right]
\end{align*}
with $a=\qty{12}{\centi\meter}$, $b=\qty{9 }{\centi\meter}$ and $c=\qty{5}{\centi\meter}$.

\end{document}